\documentclass[pdflatex,sn-mathphys-num]{sn-jnl}

\usepackage{graphicx}%
\usepackage{multirow}%
\usepackage{amsmath,amssymb,amsfonts}%
\usepackage{amsthm}%
\usepackage{mathrsfs}%
\usepackage[title]{appendix}%
\usepackage{xcolor}%
\usepackage{textcomp}%
\usepackage{manyfoot}%
\usepackage{booktabs}%
\usepackage{algorithm}%
\usepackage{algorithmicx}%
\usepackage{algpseudocode}%
\usepackage{listings}%

\theoremstyle{thmstyleone}%
\newtheorem{theorem}{Theorem}
\newtheorem{proposition}[theorem]{Proposition}%

\newtheorem{lemma}[theorem]{Lemma}

\theoremstyle{thmstyletwo}%
\newtheorem{example}{Example}%
\newtheorem{remark}{Remark}%

\theoremstyle{thmstylethree}%
\newtheorem{definition}{Definition}%

\begin{document}

\title[Discovering Elementary Discourse Units in Textual Data Using Canonical Correlation Analysis]{Discovering Elementary Discourse Units in Textual Data Using Canonical Correlation Analysis}


\author*[1]{\fnm{Akanksha} \sur{Mehndiratta}}\email{mehndiratta.akanksha@gmail.com}

\author[1]{\fnm{Krishna} \sur{Asawa}}\email{krishna.asawa@jiit.ac.in}
\equalcont{These authors contributed equally to this work.}

\affil*[1]{\orgdiv{Department of Computer Science \& Engineering and Information Technology}, \orgname{Jaypee Institute of Information Technology}, \orgaddress{\street{A-10, Sector-62}, \city{Noida}, \postcode{201309}, \state{Uttar Pradesh}, \country{India}}}


\abstract{Canonical Correlation Analysis (CCA) has been exploited immensely for learning latent representations in various fields. This study takes a step further by demonstrating the potential of CCA in identifying Elementary Discourse Units(EDUs) that captures the latent information within the textual data.  The probabilistic interpretation of CCA discussed in this study utilizes the two-view nature of textual data, i.e. the consecutive sentences in a document or turns in a dyadic conversation, and has a strong theoretical foundation. 

Furthermore, this study proposes a model for Elementary Discourse Unit(EDU) segmentation that discovers EDUs in textual data without any supervision. To validate the model, the EDUs are utilized as textual unit for content selection in textual similarity task. Empirical results on Semantic Textual Similarity(STSB) and Mohler datasets confirm that, despite represented as a unigram, the EDUs deliver competitive results and can even beat various sophisticated supervised techniques. The model is simple, linear, adaptable and language independent making it an ideal baseline particularly when labeled training data is scarce or nonexistent.}




\maketitle

\section{Introduction}\label{sec1}

Elementary Discourse Units (EDUs) are the smallest units used to represent sentence meaning in discourse analysis, such as clauses or other small meaningful segments. EDUs form the basic building blocks of larger discourse structures, enabling a better understanding of language structure and improving a variety of NLP tasks. For most NLP-based applications, EDU identification/segmentation serve as a fundamental unit and hence there is a constant need for methods that can perform these tasks efficiently.

Canonical Correlation Analysis (CCA), on the other hand, is a multivariate statistical technique used to explore relationships between two sets of variables. CCA utilizes the multi-view representation of the same latent object to extract representation of the latent state. CCA requires data with multiple views to learn a common latent space and maps each view onto this latent space to discover view-independent representations. 

The multi-view nature of data refers to situations where information about a particular entity or phenomenon is available from multiple perspectives or sources. These views might come from different sources or represent distinct subsets of features. 
Two-view data is a specific case of multi-view representation where there are only two distinct sets of features or perspectives. For instance, in a product review system a textual review and an image of the product will provide two different set of features that may be regarded as two-views. In video analysis, there are often two-views one might be the audio track, capturing the spoken words and sounds, while the other might be the visual content, capturing the scene or facial expressions.

This study presents a probabilistic interpretation of Canonical Correlation Analysis (CCA) for learning a latent space in textual data under a two-view setting. The approach relies on identifying past and future views in textual data, such as consecutive sentences in a document or turns in a dyadic conversation, which can be perceived as two distinct views. By examining the correlation between two sentences, CCA models a latent space and learns Elementary Discourse Units (EDUs) for each sentence, based on the words and their positions in the other sentence. The goal is to identify the hidden or shared intent in these consecutive sentences or utterances using the learned EDUs. These EDUs play a critical role in content selection for modeling discourse in textual data. To the best of our knowledge, this is the first investigation that uses CCA to generate EDUs. 

The study exploits the two-view setting under a conditional independence assumption. The conditional independence assumption, represented in figure \ref{fig:HSM2}, states that the two views (\textit{a\textsubscript{1}} and \textit{a\textsubscript{2}}) are said to be conditionally independent given some hidden(latent) state (\textit{L}). The study proposed here revolves around the conditional independence assumption of the past and future views in textual data on some hidden state and the learning paradigm proposed exploits the assumption mentioned above. It is also easy for one to find applications in text and NLP, under a two-view setting, where the conditional independence assumption holds naturally. This inherent association between the two-view setting and the conditional independence assumption is the foundation for the work presented in this study.

The study also introduces a CCA-based two-view learning framework to design an EDU segmentation model that is simple, linear, and works without supervision, making the model adaptable and suitable for any language, particularly when labeled training data is limited or unavailable. Our CCA-based two-view learning framework can be considered a generalization of widely used CCA-based approaches. We extend this model by considering various informational cues within the textual data as different views, such as consecutive sentences in a document or consecutive utterances in dialogue, unlike traditional CCA, which perceives inputs from two distinct sources or modes as two views \cite{guo2021manifold,jia2019semantically}. 

To understand and analyze the role of the EDUs discovered using CCA, the study utilizes these units to perform a Textual Similarity Task. The proposed model quantifies semantic equivalence based on the units identified in both sentences of a sentence pair. These units establish a shared hidden state for the input sentence pair, which is used for content selection and generates a similarity score. Experiments on the Semantic Textual Similarity Benchmark (STSB) and Mohler datasets confirm that the proposed model is highly competitive and outperforms the state-of-the-art, particularly on small datasets.

\begin{figure}
\centering
\includegraphics{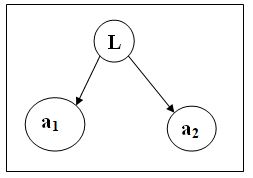}
\caption{Graphical model for Multi-view assumption adapted from \protect\cite{foster2008multi}}
\label{fig:HSM2}    
\end{figure} 

\section{Related Work}
One popular approach to discourse parsing is based on Rhetorical Structure Theory (RST)\cite{mann1988rhetorical} framework, which analyzes a document or passage by organizing its structure hierarchically. Most of the research in the direction of RST-style discourse parsing focuses on discourse segmentation and tree building. Another discourse processing framework, Segmented Discourse Representation Theory (SDRT)\cite{lascarides2007segmented}, was developed for representing and interpreting sentence meaning. SDRT emphasizes on constructing meaning across connected pieces of text or conversation using Discourse Representation Structures (DRSs) to represent meaning. DRSs are formal representations that capture the entities, conditions, and relationships introduced in a discourse and evolve as new information becomes available to the discourse. Both RST and SDRT segment the text into spans such that they cover the entire document and are non-overlapping. Therefore discourse unit segmentation task in both aims to identify the starting point of each discourse unit. 

In addition to these frameworks, the Penn Discourse TreeBank (PDTB)\cite{prasad2007penn} offers another popular framework that generates a flat discourse structure on the input document or passage. Identifying EDUs in PDTB-style discourse parsing emphasizes in identifying the spans of discourse connectives that signal a discourse relation.

Over the years, various methods have been developed to perform discourse unit/EDU segmentation to perform discourse parsing based on RST, SDRT and PDTB frameworks, but machine learning techniques specifically the use of sequential neural networks have revolutionized the field. For example, Li et al\cite{li2018segbot} proposed a model to select text boundaries using a bidirectional recurrent neural network along with a pointer network while Wang et al.\cite{wang2018toward} suggested an end-to-end neural segmenter using a BiLSTM-CRF framework. More recently, transformer-based approaches for document and discourse level segmentation were explored by Lukasik et al.\cite{lukasik2020text} and Bakshi et al. \cite{bakshi2021transformer}.
 
Various tasks in NLP perform discourse processing that focuses on spoken text. Models designed for passages or document-style discourse parsing may not be appropriate for dialogue text or spoken language. Most research in this field focuses on performing discourse parsing on dialogue systems. Popular datasets for multiparty dialogue discourse parsing include STAC \cite{asher2016discourse} and Molweni \cite{li2020molweni}. STAC corpus contains data from text-based conversations between players in online strategy games, annotated for discourse structure, including relations like contrast, elaboration, or cause-effect. The Molweni dataset also includes discourse structure annotations, labeling the relationships between different turns in the conversation, such as elaborations, questions, answers, clarifications, or follow-ups. Discourse processing using STAC or Molweni dataset centers around detecting discourse dependency links and classifying discourse relations.

Early work on the discourse parsing model for multi-party dialogue includes shallow models based on max entropy using hand-crafted features\cite{afantenos2015discourse}, and an integer linear programming-based model\cite{perret2016integer}. Shi and Huang\cite{Shi2018ADS} proposed a deep sequential model for discourse parsing that constructs a discourse structure by predicting dependency relations. Majumder et al\cite{Majumder2020InterviewLM}introduced a probabilistic framework to link dialog with facts and a generative model to perform classification on a corpus developed using media dialog. 

In recent years, much of the work in numerous NLP tasks has revolved around modeling individual sentences rather than discourse structure, as analyzing discourse structure requires an annotated dataset to examine a sentence's meaning. Large Language Models(LLMs) have shown significant improvements in this area. LLMs are trained on huge amounts of data, enabling them to understand the underlying semantics of sentences and documents, thus improving accuracy in various NLP tasks. Experimental analysis\cite{Gu2020DialogBERTDR}\cite{Santra2021RepresentationLF} supports that a thorough analysis of discourse is not necessary for creating an understanding of a sentence. 




Despite significant progress, there are several open challenges that need to be addressed to advance this field. The first challenge is building datasets. Most methods designed for discourse/EDU segmentation are data-driven and require large-scale, high-quality corpora containing EDUs, relations, and dialogue to train powerful models.  
Besides scale, ensuring annotation consistency in large datasets is also difficult. The literature indicates that discourse processing relies heavily on a meticulously developed dataset, but annotating discourse structure and relations poses a difficulty in maintaining consistency, especially when multiple or less trained annotators contribute in the development of the dataset. 

The use of sequential deep neural architectures imposes the second challenge. These black-box models makes it difficult to identify the features responsible for their high accuracy rates. Further, in NLP applications, EDUs are established using several deep neural architectures, tailored to embed knowledge relevant to the specific task/problem being addressed. This field lacks a unified structure for constructing and evaluating such EDUs. These challenges limit the applicability of discourse unit/ EDU segmentation, particularly for domain-specific datasets or datasets in other languages, as the available data is often insufficient to train deep learning-based models.
 
The main contribution of this work is to propose an EDU segmentation framework that is adaptable and language-independent. Unlike the traditional classification or prediction-based segmentation models, the proposed framework does not require the development of extensively annotated datasets. This study establishes the theoretical foundation for demonstrating that CCA can perform EDU segmentation task and the two-view setting, where we apply this interpretation of CCA, is one where obtaining unlabeled samples is easy whereas labeled samples are scarce.




\section{Canonical Correlation Analysis}
Canonical Correlation Analysis (CCA)\cite{Hotelling1936RelationsBT} is a multivariate statistical technique used to explore the relationships between two sets of variables. The primary goal of CCA is to find linear combinations of variables in each set, known as canonical variates, such that the correlation between the sets of canonical variates is maximized. In other words, CCA identifies the most highly correlated pairs, essentially a linear combination of variables in two sets.

CCA requires two sets of variables, one set for each domain. For example, you might have one set of variables related to physical activities and another set related to an individual's health. It then generates a pair of canonical variates, a linear combination of variables from each set of physical activity and individual health, that are maximally correlated. The number of canonical variates is equal to the minimum number of variables in the two sets. The variates are determined by learning of the Canonical Correlations. These represent the correlation coefficients between the sets of canonical variates. The goal is to maximize these correlations.

Consider two sets of random variables a and b, where a has m variables (features) and b has n variables. The goal of CCA is to find linear combinations of a and b such that the correlation between these linear combinations is maximized. Assuming that a and b are jointly Gaussian (multivariate normal) random variables
\begin{equation} \label{eq:1}
 \begin{aligned}
 (a,b) \sim \mathcal{N}(\mu\textsubscript{a}, C\textsubscript{aa}) \otimes \mathcal{N}(\mu\textsubscript{b}, C\textsubscript{bb}) \\
 \end{aligned}
 \end{equation}
Where $\mathcal{N}(\mu\textsubscript{a}, C\textsubscript{aa})$ is the Gaussian distribution of a with mean $\mu\textsubscript{a}$ and covariance matrix $C\textsubscript{aa}$. $\mathcal{N}(\mu\textsubscript{b}, C\textsubscript{bb})$ is the Gaussian distribution of b with mean $\mu\textsubscript{b}$ and covariance matrix $C\textsubscript{bb}$. $\otimes$ denotes the Kronecker product. CCA finds vectors U\textsubscript{1} and U\textsubscript{2} such that the canonical correlation, denoted by $\rho$, is maximized. The canonical variables are given by 
\begin{equation} \label{eq:2}
 \begin{aligned}
 \lambda\textsubscript{a} = a U\textsubscript{1}  \\
 \lambda\textsubscript{b} = b U\textsubscript{2} \\
 \end{aligned}
 \end{equation} where U\textsubscript{1} and U\textsubscript{2} are the canonical weight vectors. CCA seeks U\textsubscript{1} and U\textsubscript{2} to maximize the correlation $\rho$ between $\lambda$\textsubscript{a} and $\lambda$\textsubscript{b}

 \begin{equation} \label{eq:3}
 \begin{aligned}
 \rho = Corr(\lambda\textsubscript{a}, \lambda\textsubscript{b}) = \frac{Cov(\lambda\textsubscript{a}, \lambda\textsubscript{b})}{\sqrt{Var(\lambda\textsubscript{a}) \cdot Var(\lambda\textsubscript{b})}} \\
 \end{aligned}
 \end{equation}

 The optimization problem in CCA is framed around maximizing the canonical correlation $\rho$, which is formulated as an eigenvalue problem involving the cross-covariance matrices $C\textsubscript{ab}$ and $C\textsubscript{ba}$

\begin{equation} \label{eq:4}
\begin{aligned}
 \underset{U\textsubscript{1}, U\textsubscript{2}}{\text{max $\rho$}} = \underset{U\textsubscript{1}, U\textsubscript{2}}{\text{max}} \frac{U\textsubscript{1}^T C\textsubscript{ab} U\textsubscript{2}} { \sqrt{U\textsubscript{1}^T C\textsubscript{aa} U\textsubscript{1} \cdot U\textsubscript{2}^T C\textsubscript{bb} U\textsubscript{2}}}\\
\end{aligned}
\end{equation}

The outcome of the optimization is a diagonal matrix containing canonical correlations $\rho$ = diag([p\textsubscript{0}, ..., p\textsubscript{dim}]), where dim = min(m,n). The projections are maximally correlated if i = j, with correlation coefficient p\textsubscript{i}, and uncorrelated otherwise. 

\section{Suitability of CCA in modeling EDUs}
CCA is a powerful tool for identifying and understanding relationships between sets of variables. It provides a way to extract meaningful patterns and associations from multivariate data. 

Recently, Bach and Jordan\cite{article} demonstrated that when considering two random vectors that are independent conditional on some hidden state L, as shown in figure \ref{fig:HSM1}, CCA constitutes as an effective tool in generating an interpretation of L. The probabilistic interpretation of CCA assumes that the data follows a joint Gaussian distribution, and the optimization problem aims to identify the most highly correlated pairs, essentially a linear combination of variables within each vector. Canonical correlation measures the strength of the linear association between the canonical variables obtained from a and b.

Similarly, Foster et al \cite{foster2008multi} capitalized on CCA, while working on input data under a multi-view setting, to lower the dimensionality of the input vector. Motivated by the work presented by Bach and Jordan\cite{article}, Foster et al\cite{foster2008multi} approached the model in figure \ref{fig:HSM1} from a two-view perspective.  Foster et al\cite{foster2008multi} replaced the two random variables with two views a\textsubscript{1} and a\textsubscript{2} of input data and interpreted the probabilistic model as a conditional independence assumption, illustrated in the figure \ref{fig:HSM2}, which implies that
\begin{align}
    Prob(a\textsubscript{1}, a\textsubscript{2}\text{\textbar} L) =Prob (a\textsubscript{1}\text{\textbar} L) Prob (a\textsubscript{2}\text{\textbar} L)
\end{align}

\begin{figure}
\centering
\includegraphics{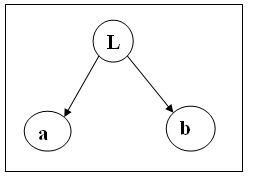}
\caption{Graphical model for Latent Interpretation Model adapted from \protect\cite{article}}
\label{fig:HSM1}    
\end{figure} 

\subsection{Theoretical Foundation}
We first discuss the lemma that highlights the role of CCA in determining the hidden state. Given two random variables a and b, Bach and Jordan\cite{article} proposed an interpretation of hidden state using CCA. 

 \begin{lemma}\label{lemma1}
We consider the model given in figure \ref{fig:HSM1} where a and b are two random variables of dimensions m and n respectively defined by \ref{eq:5}
\begin{equation} \label{eq:5}
 \begin{aligned}
L \sim \mathcal{N}(0,I\textsubscript{dim}) \\
a \text{\textbar} L \sim \mathcal{N}(W\textsubscript{a} L + \mu\textsubscript{a},\Psi\textsubscript{a}) \\
b \text{\textbar} L \sim \mathcal{N}(W\textsubscript{b} L + \mu\textsubscript{b},\Psi\textsubscript{b})
 \end{aligned}
 \end{equation}
 Here L $\in\mathbb{R}$\textsuperscript{(dim)} is the shared hidden state, $I$ is an Identity matrix and $dim$ = min(m, n) is the dimension of the space onto which the views are projected. The maximum likelihood estimates of the model parameters W\textsubscript{a} $ \in\mathbb{R}$\textsuperscript{(m $\times$ dim)}, W\textsubscript{b} $ \in\mathbb{R}$\textsuperscript{(n $\times$ dim)}, $\mu$\textsubscript{a}, $\mu$\textsubscript{b}, $\Psi$\textsubscript{a} and $\Psi$\textsubscript{b} are derived as first dim canonical directions as \ref{eq:6}
\begin{equation}  \label{eq:6}
 \begin{aligned}
 \Hat{W\textsubscript{a}} = C\textsubscript{aa} U\textsubscript{a} M\textsubscript{a} \\
 \Hat{W\textsubscript{b}} = C\textsubscript{bb} U\textsubscript{b} M\textsubscript{b} \\
 \Hat{\Psi\textsubscript{a}} = C\textsubscript{aa} - \Hat{W\textsubscript{a}} \Hat{W\textsubscript{a}}\textsuperscript{T}\\
\Hat{\Psi\textsubscript{b}} = C\textsubscript{bb} - \Hat{W\textsubscript{b}} \Hat{W\textsubscript{b}}\textsuperscript{T}\\
\Hat{\mu\textsubscript{a}}= \mu\textsubscript{a} \\
\Hat{\mu\textsubscript{b}}= \mu\textsubscript{b} 
 \end{aligned}
 \end{equation}
 where C is the covariance matrix, the arbitrary matrices M\textsubscript{a} and M\textsubscript{b}, $ \in\mathbb{R}$\textsuperscript{(dim $\times$ dim)} such that M\textsubscript{a}M\textsubscript{b}\textsuperscript{T} = $\rho$\textsubscript{dim} and the spectral norms be smaller than one. The j\textsuperscript{th} column of matrices U\textsubscript{a} and U\textsubscript{b} are equal to the j\textsuperscript{th} eigenvector and the first dim canonical correlations are given by a diagonal matrix $\rho$\textsubscript{dim}.   
 \end{lemma}
\begin{proof}
\ref{app1}    
\end{proof}
 The lemma \ref{lemma1} is based on a probabilistic model which implies that a and b are independent variables conditional on some hidden state. Assuming that the hidden space follows Gaussian distribution, the lemma \ref{lemma1} supports that CCA unearths the shared hidden state on a lower dimensional space.
 
 The two-view framework and the interpretation, proposed by Bach and Jordan\cite{article}, attracted significant attention. Dhillon et al\cite{dhillon2015eigenwords} utilized this setting to generate semantic word embeddings by defining one view as a specified number of words surrounding the target word, while considering the target word itself as the other view.
Foster et al.\cite{foster2008multi} on the other hand, utilized this setting to analyze the multi-view regression problem. 
 Referencing the two random variables a and b in the graphical model representation given by figure \ref{fig:HSM1} as two views of input data and the probabilistic model as an assumption, called as conditional independence, Foster et al.\cite{foster2008multi} showed that CCA can lower the dimensionality of the input data (by generating a latent state) without losing its predictive power.

Let us define the operation of CCA as \ref{eq:7} 
\begin{equation}  \label{eq:7}
    (\lambda\textsubscript{a\textsubscript{1}},  \lambda\textsubscript{a\textsubscript{2}}) = CCA(a\textsubscript{1}, a\textsubscript{2})
\end{equation}
Here a\textsubscript{1} and a\textsubscript{2} are two-views of the input data a. $\lambda$\textsubscript{a\textsubscript{1}} and $  \lambda$\textsubscript{a\textsubscript{2}} are the projection of a\textsubscript{1} and a\textsubscript{2} respectively outputted by CCA onto a dim = min(m,n) dimensional subspace. We now discuss the lemma that shows that the projections have a reduced dimensionality still their predictive power of the target variable remains unaffected.
\begin{lemma}\label{lemma2}
    Considering that the conditional independence holds for the model given in figure \ref{fig:HSM2} and that the dimension of L is dim. Then $\lambda$ represents the CCA subspace of dimension dim and
    \begin{enumerate}
        \item the best linear estimator of target variable Z with a\textsubscript{1} as well as its projection $\lambda$\textsubscript{a\textsubscript{1}} are equal.

        \item the best linear estimator of target variable Z with a\textsubscript{2}  as well as its projection $\lambda$\textsubscript{a\textsubscript{2}} are equal.
    \end{enumerate}
\end{lemma}
\begin{proof}
\ref{app2}    
\end{proof}

It is evident from lemma \ref{lemma2} that the representations generated by CCA for each sentence are equivalent to the input sentence-pair. 

Although lemma \ref{lemma1} provides a probabilistic interpretation of CCA for determining a hidden state that consists of low dimensional real vectors, it is lemma \ref{lemma2} that highlights the role of these representations in replacing the input data thus justifying its designation as an elementary discourse unit.

\section{Modeling EDUs Segmentation using CCA-based two-view learning framework}
The goal of this study is to estimate a set of vectors for each sentence that captures the shared hidden intent in that sentence in the form of latent representations of the correlation between the word of the sentence and the word of the sentence in its immediate context.

More formally, assume a document or a conversation consisting of n sentences {s\textsubscript{1}, s\textsubscript{2}, ..., s\textsubscript{n}}. Define each sentence as a list of the embedding, where a k-dimensional vector is determined to represent each word using Glove \cite{pennington2014glove}. s\textsubscript{i} = (word\textsubscript{i1}, word\textsubscript{i2}, ..., word\textsubscript{im}), i = 1, 2,..., m, where each element, word\textsubscript{i1}, is the k-dimensional embedding equivalent of its corresponding word. The model projects the sentence onto a hidden space to generate latent representations of dimension (dim,k) that captures the hidden intent by applying CCA on the multivariate random variables s\textsubscript{i} and s\textsubscript{i} of dimension (m,k) and (n,k) respectively. Here dim = min(m , n). Therefore, each latent representation generated acts as a EDU and can be mapped to a word using Glove. 
The EDU segmentation model is given by algorithm \ref{alg:edu}
    \begin{algorithm}
        \caption{EDU Segmentation using CCA based framenwork}\label{alg:edu}
        \begin{algorithmic}[1]
        \While{$N \neq size$}
        \State num\textunderscore of\textunderscore EDUs = min(len(s\textsubscript{1}), len(s\textsubscript{2}))
        \State cca = CCA(n\textunderscore components=num\textunderscore of\textunderscore EDUs)
        \State Fit(s\textsubscript{1},s\textsubscript{2}) 
        \State EDU\textsubscript{1}, EDU\textsubscript{2} = Transform(s\textsubscript{1}, s\textsubscript{2})
        \State N $gets$ N + 1
\EndWhile
        \end{algorithmic}
        \end{algorithm}

CCA is a function defined in SKLearn, an open-source Python library. Method Fit(A, B) is tasked to fit the model to input data. On the other hand, Transform(A,[B, copy]) transforms the input A and B by maximizing the correlation amongst them and returns the latent variable pairs (via projections) EDU\textsubscript{A}, EDU\textsubscript{B} respectively. 

\section{Evaluating the EDUs} 
Semantic Textual Similarity (STS) refers to the task of determining the degree of similarity between two pieces of text based on their underlying meaning. The fundamental concept in the design of STS models is to extract semantic information, understand the relationship between them and output a score based on the nature of the relationship. This core concept behind STS serves as a cornerstone for various tasks in NLP, such as document summarization, textual entailment and many more.

To strengthen and validate the learned EDUs under a two-view setting in textual data, experiments were performed on three textual similarity tasks from SemEval Semantic Textual Similarity (STS) challenges (2012-2017). Although STS tasks are considered standardized datasets for textual similarity, the model was also tested on the Automatic Short Answer Grading(ASAG) task. ASAG is used to assess and grade short answers submitted by students. These systems analyze the content, structure, and relevance of the answers and provide a score. The ASAG task resembles the semantic textual similarity task as both tasks output a degree of similarity between two pieces of text. The ASAG task was chosen as it typically involves a small and domain-specific dataset that is not enough to provide training in deep learning based models. 
\subsection{Datasets}

\subsubsection{STS Dataset} The dataset considered for experimenting was SemEval-2017 Task 1\cite{cer-etal-2017-semeval} also known as STSB. STSB dataset consists of the data published in the SemSval English STS shared tasks (2012 - 2017). The dataset contains 8628 sentence-pairs. Each sentence pair in these datasets is accompanied by a similarity score that ranges from 0(highly dissimilar) to - 5(highly similar) as shown in Table 1.
\begin{table}
\caption{A Sample demonstrating the sentence pair STS-B dataset.}\label{tab1}
\begin{tabular}{|p{2.5cm}|p{4.5cm}|p{4.5cm}|}
\hline
 &  {\bfseries Example - 1} & {\bfseries Example - 2} \\
\hline
{\bfseries Sentence 1} &  A plane is taking off.
 & Two boys are driving.\\ \hline
{\bfseries Sentence 2} &  An air plane is taking off. & Two bays are dancing.\\ \hline
{\bfseries Similarity Score} & 5 (The two sentences mean the same thing hence are completely equivalent)  & 0.6 (The two sentences may be around the same topic but are not equivalent)\\ 
\hline
\end{tabular}
\end{table}

\subsubsection{Mohler Dataset} The Mohler dataset\cite{Mohler2009TexttoTextSS} refers to a collection of short-answer text data that is often used in natural language processing (NLP) research and machine learning tasks, particularly in the context of ASAG. The dataset contains 80 questions along with desired answer with approximately 24 to 30 student submissions for each question. Thus containing 2273 pairs of the desired answer and student's submission. Each student submission is graded by two assigned teachers and an average grade is also attached with each pair as illustrated in table \ref{tabmohler}.

\begin{sidewaystable}
\caption{Samples demonstrating the questions, reference answers, student answers, and the grades awarded from the Mohler's dataset.}\label{tabmohler}

\centering
\begin{tabular}{|p{0.5cm}|p{3.5cm}|p{3.5cm}|p{7cm}|p{1.5cm}|p{1.5cm}|p{1cm}|}
\hline
 {\bfseries Id} & {\bfseries Question} & {\bfseries Desired Answer} &{\bfseries Student Answer} & {\bfseries Teacher 1 Grade} & {\bfseries Teacher 2 Grade} & {\bfseries Average} \\ \hline
1.1 & What is the role of a prototype program in problem solving? &  To simulate the behaviour of portions of the desired software product. & High risk problems are address in the prototype program to make sure that the program is feasible. & 4 & 3 & 3.5\\ \hline

1.1 & What is the role of a prototype program in problem solving? &  To simulate the behaviour of portions of the desired software product. & To simulate portions of the desired final product with a quick and easy program that does a small specific job. It is a way to help see what the problem is and how you may solve it in the final project. & 5 & 5 & 5\\ \hline

1.1 & What is the role of a prototype program in problem solving? &  To simulate the behaviour of portions of the desired software product. & High risk problems are address in the prototype program to make sure that the program is feasible.  A prototype may also be used to show a company that the software can be possibly programmed. & 4 & 3 & 3.5\\ \hline


\end{tabular}
\end{sidewaystable}

\subsection{Modeling Textual Similarity}
Here's how this study exploits the EDUs generated using CCA in developing an application that performs the similarity task. In this setting, the given two pieces of text a\textsubscript{1} and a\textsubscript{2} are perceived as two views of input a = (a\textsubscript{1}, a\textsubscript{2}) and the algorithm to compute the semantic similarity score s given by algorithm \ref{alg:sts}
    \begin{algorithm}
        \caption{Modeling Semantic Textual Similarity Using EDUs}\label{alg:sts}
        \begin{algorithmic}[1]
        \State Using a = (a\textsubscript{1}, a\textsubscript{2}), perform CCA.
        \State Construct the projections EDU\textsubscript{a\textsubscript{1}} and EDU\textsubscript{a\textsubscript{2}} of size dim.
        \State Design a distance function f(EDU\textsubscript{a\textsubscript{1}}, EDU\textsubscript{a\textsubscript{2}}) to formulate a similarity score.
        \end{algorithmic}
        \end{algorithm}

Given a sentence pair (a\textsubscript{i}, a\textsubscript{j}), CCA outputs projections that are essentially a linear transformation of a\textsubscript{i} and a\textsubscript{j}. This function returns the canonical pairs of direction, in other words, a projection matrix, for each a\textsubscript{1} and a\textsubscript{2} respectively. This outputted projection matrix is then used to transform a\textsubscript{1} and a\textsubscript{2} to develop EDU\textsubscript{a\textsubscript{1}} and EDU\textsubscript{a\textsubscript{2}} respectively.

Here the number of projections determined by CCA is given as min(m, n) where m and n are the length of a\textsubscript{i} and a\textsubscript{j}. Table~\ref{tab2} demonstrates canonical variables/ projections determined by CCA on a few samples of desired answer-student submission pair from the Mohler dataset. The EDUs in the table are listed as (w\textsubscript{i1} , w\textsubscript{i2} , ... , w\textsubscript{in} ) , i = 1 , 2 , . . . , n, where each element is the word embedding counterpart of the respective projection identified by CCA.


\begin{table}
\centering
\caption{Projections and Similarity score obtained from the proposed model on a few samples of desired answer-student submission pair.}
\label{tab2}
\begin{tabular}{|p{5cm}|p{3cm}|p{1.5cm}|p{1.5cm}|}
\hline
{\bfseries Sentence Pair} & {\bfseries Projection} & {\bfseries Average Score} & {\bfseries Proposed Model} \\ \hline
 {\bfseries Desired Answer} 'To simulate the behaviour of portions of the desired software product.' & ['software', 'behaviour', 'product', 'portions', 'desired'] & 2 & 1.9 
 \\

 {\bfseries Student Answer} 'To find problem and errors in a program before it is finalized.'& ['program', 'problem', 'find', 'finalized', 'errors'] & &\\
\hline

 {\bfseries Desired Answer} At the main function.	&	['main', 'function'] &	5 &	3.78
 \\

{\bfseries Student Answer} The main method. & ['main', 'method'] & &	 

  \\\hline
  {\bfseries Desired Answer}  A location in memory that can store a value. &	['store', 'value', 'memory', 'location'] & 3.5 & 4.3 \\
  {\bfseries Student Answer} An object with a location in memory where value can be stored & ['value', 'stored', 'location', 'memory'] & &\\ \hline
 {\bfseries Desired Answer} The block inside a do...while statement will execute at least once. & ['statement', 'least', 'execute', 'inside', 'block'] & 5 & 2.69 \\
 {\bfseries Student Answer} a while statement will only process if the statement is met, while a do...while will always process once, then only continue if the statement is met. & ['statement', 'always', 'process', 'always', 'process'] & & \\ \hline 

\end{tabular}
\end{table}

\subsection{Formulating Similarity Score}
To generate a score for similarity, cosine similarity is applied between the projection pairs. The average cosine similarity between two projection pairs is then scaled to 5 to determine a final similarity score.

\subsection{Results and Analysis}


The effectiveness of a model in such a setting is computed using the Pearson correlation coefficient between the scores determined by the model and the scores provided in the dataset. The higher the Pearson’s r the better the model. Table \ref{tabstsb} shows the result of various models experimented on the Semeval-2017(STSB) dataset, and table \ref{tabmohlerpearson} on the Mohler dataset. For STS datasets, large language models significantly outperform on the standardized datasets with a Pearson score of around 0.90 but the model proposed is competitive with deep sequential models. With Mohler dataset the Pearson’s r score shows that this simple linear model significantly outperforms the grading accuracy of various non-linear complex deep learning models. 

\begin{table}
\caption{Experiment on the STS-B task.}\label{tabstsb}
\centering
\begin{tabular}{|p{6cm}|p{1.5cm}|}
\hline
{\bfseries Model} &	 {\bfseries Pearson's r Score} \\
\hline

BiLSTM\cite{wang-etal-2018-glue} &	0.660\\ \hline
+CoVe\cite{mccann2017learned} &	0.672\\ \hline
BiLSTM\cite{wang-etal-2018-glue}	& 0.703\\ \hline
+Attn, ELMo\cite{wang-etal-2018-glue} &	0.742\\ \hline
Infersent\cite{wang-etal-2018-glue} &	0.759\\ \hline
GenSen\cite{wang-etal-2018-glue} &	0.793\\ \hline
Proposed Model &	0.797 \\  \hline
Bert base\cite{dasgupta-etal-2023-cost} & 0.898  \\  \hline
DistilBERT\cite{dasgupta-etal-2023-cost}(layers=6 x hidden state=768) & 0.810  \\  \hline
TinyBERT\cite{dasgupta-etal-2023-cost}(layers=6 x hidden state=368) & 0.837  \\  \hline

\end{tabular}
\end{table}

\begin{table}
\centering
\caption{Experiment on the Mohler dataset.}
\label{tabmohlerpearson}
\begin{tabular}{|p{4cm}|p{3cm}|}
\hline
 {\bfseries Model} & {\bfseries Pearson's r Score} \\ \hline
 Bi-LSTM-Capsule \cite{Zhu2022AutomaticSG} & 0.507 \\ \hline
 Bi-LSTM + CNN \cite{Zhu2022AutomaticSG} & 0.517 \\ \hline 
 CNN \cite{Zhu2022AutomaticSG} & 0.002 \\ \hline
 Capsule \cite{Zhu2022AutomaticSG} & 0.070 \\ \hline
 Bi-LSTM \cite{Zhu2022AutomaticSG} & 0.092
\\ \hline
ELMo \cite{Gaddipati2020ComparativeEO} & 0.485 \\ \hline
GPT \cite{Gaddipati2020ComparativeEO} & 0.248 \\ \hline
BERT \cite{Gaddipati2020ComparativeEO} & 0.318 \\ \hline
GPT-2 \cite{Gaddipati2020ComparativeEO} & 0.311 \\ \hline
Proposed Model & 0.512 \\ \hline
\end{tabular}
\end{table}

\section{Conclusion}
We have presented an automated EDU segmentation framework using CCA that is simple, understandable, and easily integrated into most NLP and text-based applications. Notably, it eliminates the need for large labeled datasets, making various NLP tasks more accessible. The framework does not pose any language limitations. Since it is adaptable, numerous extensions are possible for this linear probabilistic model. While its contribution to the similarity task may seem marginal, considering that the method is linear, adaptive, and supports unlabeled datasets, it shows significant potential and opens up a wide range of applications where it can be applied. 

Deep-learning architectures are a popular choice for achieving high accuracy in most segmentation tasks, consequently increasing the need for large labeled data sets and resource-intensive training. On the other hand, the framework proposed is simple and linear, yet, it certainly stands out and shows promise compared to the state-of-the art.

\section{Future Work}
Based on the interpretation of CCA, the model proposed automatically obtains EDUs from the text, which are essentially represented by a word or a unigram. Therefore, exploring variants to discover the syntactic patterns or phrases within a sentence would be an important development. To further the discourse processing task, with this model as the foundation, it is necessary to address the following: (1) selection of the EDUs, and (2) the aggregation of selected EDUs into a structure.

\section*{Statements and Declarations}
\begin{itemize}
\item Funding: This research received no specific grant from any funding agency in the public, commercial, or not-for-profit sectors.
\item Conflict of interest/Competing interests: The authors declare that they have no financial or non-financial competing interests.
\item Ethics approval and consent to participate: Not Applicable

\item Data availability: The data used to support the findings of this study is available publicly.

\item Author contribution: All authors contributed equally to this work. Krishna Asawa contributed to defining the research question and the experimental design together with her Ph.D. student Akanksha Mehndiratta. Both authors have made substantial contributions to the conception of this study and the design of the presented model. Akanksha Mehndiratta under the supervision of Krishna Asawa performed data acquisition, implementation of the model, analysis of the result and wrote the first draft of the manuscript. Krishna Asawa performed a critical revision of the manuscript for important intellectual content and approved the final version of the manuscript.
\end{itemize}

\begin{appendices}

\section{Proof For Lemma 1}\label{app1}
\begin{proof}
    Given an input I = (a, b) of dimension m and n respectively, the marginal mean for I is given as \[ \mu = \begin{pmatrix} \mu\textsubscript{a}\\
\mu\textsubscript{b} \end{pmatrix}\] and the   covariance matrix as \[C = \begin{pmatrix} W\textsubscript{a}W\textsubscript{a}\textsuperscript{T} + \Psi\textsubscript{a} & W\textsubscript{a}W\textsubscript{b}\textsuperscript{T}\\
W\textsubscript{b}W\textsubscript{a}\textsuperscript{T} & W\textsubscript{b}W\textsubscript{b}\textsuperscript{T} + \Psi\textsubscript{b}\end{pmatrix}\]. For input data I\textsuperscript{j}=(a\textsuperscript{j} ,b\textsuperscript{j}),j = 1,2,...k, the negative log likelihood is given as

\[
\begin{aligned}
l\textsubscript{1} = \frac{k(m + n)} { 2 } log 2\pi + \frac{k}{2} log |C| + \frac{1}{2} \sum_{j=1} ^{k} tr C\textsuperscript{-1} (I\textsuperscript{j} - \mu) (I\textsuperscript{j} - \mu)^T\\
= \frac{k(m + n)} { 2 } log 2\pi + \frac{k}{2} log |C| + \frac{k}{2} tr C\textsuperscript{-1} \Tilde{C} +  \frac{k}{2}(\Tilde{\mu} - \mu)^T C\textsuperscript{-1} (\Tilde{\mu} - \mu)\\
\end{aligned}
\]
Let us first maximize with respect to $\mu$. The maximum is obtained at sample mean ($\Tilde{\mu}$) . Updating this value in the log likelihood results in
\[
\begin{aligned}
l\textsubscript{1} = \frac{k( m + n)} { 2 } log 2\pi + \frac{n}{2} log |C| + \frac{n}{2} tr C\textsuperscript{-1} \Tilde{C}
\end{aligned}
\]
The rest of the proof follows immediately along the line of proof given by Bach and Jordan\cite{article}.
\end{proof}

\section{Proof For Lemma 2}\label{app2}
\begin{proof}
    Let the covariance between two variables be denoted as 
     \begin{align*}
     C\textsubscript{11} &= \mathop{\mathbb{E}}[a\textsubscript{1}(a\textsubscript{1})\textsuperscript{T}]\\
     C\textsubscript{22} &= \mathop{\mathbb{E}}[a\textsubscript{2}(a\textsubscript{2})\textsuperscript{T}]\\
    C\textsubscript{12} &= \mathop{\mathbb{E}}[a\textsubscript{1}(a\textsubscript{2})\textsuperscript{T}]
 \end{align*}

Let L be a linear latent state, then the covariance between two variables a\textsubscript{1} and a\textsubscript{2}, similar to the proof presented by Foster et al.\cite{foster2008multi}, is given as
\begin{align*}
C\textsubscript{12} = C\textsubscript{1L}C\textsubscript{L2}    
\end{align*}
Hence the following equality holds
\begin{align*}
C\textsubscript{1L} = C\textsubscript{12}(C\textsubscript{L2})\textsuperscript{-1}
\end{align*}
Now, the optimal linear predictor $\beta$ is given as
\begin{equation}
\begin{split}
\beta &= C\textsubscript{1L} \\
& = C\textsubscript{12}(C\textsubscript{L2})\textsuperscript{-1}     
\end{split}
\end{equation}
Hence
\begin{align*}
\beta a\textsubscript{1} = (C\textsubscript{2L})\textsuperscript{-1} C\textsubscript{21} a\textsubscript{1} \\
\end{align*}
Let the singular value decomposition of C\textsubscript{12} be:
\begin{align*}
   C\textsubscript{12} = U\textsubscript{1} \rho U\textsubscript{2}^T 
\end{align*}

here $\rho$ is diagonal with canonical directions and the column vector of U\textsubscript{1} and U\textsubscript{2} form the CCA basis. Plugging the values in 
\begin{equation}\label{beta}
\beta a\textsubscript{1} = (C\textsubscript{2L})\textsuperscript{-1} U\textsubscript{2} \rho U\textsubscript{1}^T a\textsubscript{1} \\
\end{equation}
Since $p\textsubscript{i} = 0$ where $i \neq j$ hence
\begin{align*}
    \rho U\textsubscript{1}^T a\textsubscript{1} = \rho U\textsubscript{1}^T \lambda\textsubscript{a1}
\end{align*}
Plugging this in \ref{beta}
\begin{equation}
\begin{split}
  \beta a\textsubscript{1} &= (C\textsubscript{2L})\textsuperscript{-1} U\textsubscript{2} \rho U\textsubscript{1}^T \lambda\textsubscript{a1} \\
& = \beta \lambda\textsubscript{a1} 
\end{split}
\end{equation}
Hence proving that the claim made in 1 is valid.The proof for claim 2 follows along the same lines as for claim 1.
\end{proof}
\end{appendices}


\bibliography{sn-bibliography}

\end{document}